\documentclass[lettersize,journal]{IEEEtran}
\IEEEoverridecommandlockouts

\usepackage{amsmath,amsfonts}
\usepackage{amsthm}
\usepackage{array}
\usepackage[caption=false,font=normalsize,labelfont=sf,textfont=sf]{subfig}
\usepackage{textcomp}
\usepackage{stfloats}
\usepackage{url}
\usepackage{verbatim}
\usepackage{graphicx}
\usepackage{cite}
\usepackage{xcolor}
\usepackage{hyperref}
\usepackage{url}
\usepackage{here}
\usepackage{algorithm}
\usepackage{amssymb} 

\newtheorem{proposition}{Proposition}
\newtheorem{corollary}{Corollary}

\newtheorem{lemma}{Lemma}
\newtheorem*{remark}{Remark}
\usepackage[noend]{algpseudocode}
\algrenewcommand\algorithmicrequire{\textbf{Initiate}}
\usepackage{tabularx}
\usepackage[bottom]{footmisc}
\usepackage{comment}
\usepackage{nomencl}
\newcommand{\etal}{\emph{et~al.\vspace{-0.25px}}\@}
\usepackage{adjustbox}
\usepackage{subcaption}
\usepackage{overpic}
\usepackage{tabularx}
\usepackage{array}
\newcolumntype{Y}{>{\raggedright\arraybackslash}X}
\usepackage{graphicx}
\usepackage{overpic}
\usepackage{xcolor}
\usepackage{subcaption}
\usepackage{enumitem}
\definecolor{darkgreen}{RGB}{0,100,0}
\begin{document}
\title{Haptic Rendering of Fractional-Order Viscoelasticity: Passivity and Rendering Fidelity}

\author{Gorkem Gemalmaz,~\IEEEmembership{Student Member,~IEEE}
\hspace{25mm} Harun Tolasa,~\IEEEmembership{Student Member,~IEEE} \\ \vspace{3mm}
Volkan Patoglu,~\IEEEmembership{Member,~IEEE}
\thanks{G. Gemalmaz, H. Tolasa, and V. Patoglu are with the Faculty of Engineering and Natural Sciences at Sabanci University,  Istanbul, Turkiye.\\ {\tt\scriptsize \{gorkem.gemalmaz,harun.tolasa,volkan.patoglu\}@sabanciuniv.edu}}}


\maketitle

\begin{abstract}
Haptic rendering of viscoelastic materials that exhibit creep and stress relaxation is crucial for many applications, such as medical training with realistic biological tissue models.
Fractional-order viscoelastic models provide an effective means of describing intrinsically time-dependent dynamics with few parameters, as these models can naturally capture memory effects.
In this study, we present analyses of passivity and rendering performance for fractional-order viscoelastic models under finite-memory discretization. We derive closed-form expressions to ensure the passivity of haptic rendering with a fractional-order~(FO) standard linear solid~(SLS) model based on Grünwald-Letnikov derivative under short-memory discretization. We also provide symbolic expressions for the effective stiffness and damping of such FO-SLS models. The resulting passivity conditions constitute a unified framework that generalizes previously reported results for integer-order Kelvin–Voigt, Maxwell, and SLS models, since these results are special cases of the newly derived condition. Furthermore, we provide experimental validations of the theoretical passivity bounds and human-subject evaluations of perceived realism of FO-SLS models. Overall, this study establishes a unified theoretical framework and experimental evaluations for FO viscoelastic rendering under short-memory discretization.
\end{abstract}

\vspace{-3mm}

\section{Introduction}

\IEEEPARstart{H}{aptic} rendering allows users to physically interact with virtual objects by providing kinesthetic and tactile feedback through a haptic interface. Applications of this technology are critical to a wide variety of areas, including medical and surgical training simulators~\cite{salisbury1995haptic}. In medical training, haptic simulators allow trainees to practice complex procedures in realistic, risk-free environments, while in robot-assisted surgery, force feedback can compensate for the loss of direct tactile sensation inherent to teleoperated systems. Beyond training, the mechanical response of biological tissues provides valuable diagnostic cues, since abnormalities such as tumors often manifest through changes in viscoelastic behavior~\cite{breast_cancer,breast_tissue,prostate}.

Accurate haptic rendering of viscoelastic materials requires models that can capture the intrinsically time-dependent mechanical behavior of biological tissues. Unlike purely elastic materials, biological tissues exhibit viscoelastic phenomena such as creep and stress relaxation, which arise from their internal memory effects~\cite{time_dependant_tissue}. Traditional integer-order~(IO) spring--damper networks often fail to reproduce these behaviors faithfully, as they lack the ability to represent long-term memory and frequency-dependent dynamics~\cite{IO_fails_memory}. As a result, fractional-order~(FO) viscoelastic models have gained increasing attention as more accurate and compact representations. By generalizing classical constitutive laws through non-integer derivatives, FO models can naturally capture memory-dependent dynamics using relatively few parameters~\cite{FO_few_param,BagleyTorvik1983,Mainardi2010}.

Despite these advantages, the adoption of fractional-order models in real-time haptic rendering remains limited. A primary challenge lies in interpreting and tuning FO model parameters. While fractional-order models admit physical interpretations, the coupling between the fractional order and the remaining parameters is frequency-dependent, making the influence of each parameter on perceived material behavior difficult to predict. Consequently, identifying parameter sets that yield high perceived realism during haptic interaction is nontrivial, particularly when aiming for solutions that generalize across users with differing perception and interaction styles. This challenge has motivated recent work on human-in-the-loop~(HiL) optimization and population-level perceptual modeling, enabling the systematic identification of realistic parameter sets for FO viscoelastic models~\cite {Tolasa2026}.

Beyond parameter selection, another fundamental challenge arises due to the discrete-time implementation of fractional-order dynamics. Ideal fractional derivatives possess infinite memory, whereas real-time haptic rendering necessarily relies on discrete-time implementations with \emph{finite memory}. While continuous-time operators, such as $s^{\alpha}$, are relatively straightforward to analyze, their discrete-time realizations, typically based on the Grünwald-Letnikov definition, require truncation of the memory kernel. The resulting short-memory discretization introduces a large discrete parameter space and significantly complicates both analytical characterization and implementation of these models. Moreover, truncation length cannot be chosen arbitrarily small, as viscoelastic effects only emerge when a sufficiently long interaction history is retained.

Ensuring coupled stability during user interactions further compounds these difficulties. In haptic systems, coupled stability is commonly guaranteed through passivity, which provides an implementation-independent criterion to prevent undesired energy generation. While passivity conditions are well established for classical IO spring-damper systems~\cite{colgate1994factors,colgate}, and some results exist for continuous-time FO models~\cite{tokatli2015stability,tokatli_tez,tokatli_patoglu_2018, Aydin2017, aydin2018stable, Yusuf2020}, analytical passivity bounds for discrete-time FO viscoelastic models under finite-memory discretization are absent from the literature. In particular, no closed-form passivity conditions are available for FO standard linear solid~(FO-SLS) model implemented with realistic short-memory approximations.

To address these gaps, this study presents a unified theoretical and experimental framework for FO viscoelastic rendering under finite memory discretization. On the analytical side, we derive closed-form passivity conditions that guarantee coupled stability of FO-SLS virtual environments implemented under short-memory discretization. We further characterize the rendering behavior of these models by deriving symbolic expressions for their effective stiffness and effective damping. These expressions not only facilitate performance analyses but also reveal how classical IO Kelvin-Voigt, Maxwell, and SLS models emerge as special cases of the proposed framework. On the perceptual side, we experimentally compare the perceived realism of three representative FO-SLS virtual environments with varying memory truncation lengths against a physical viscoelastic material, to study the effect of memory length on perceived realism in viscoelastic rendering. Overall, this study provides a comprehensive assessment of both coupled stability and performance of FO viscoelastic rendering, and establishes a principled pathway for the reliable and practical use of FO models in interaction control.


\medskip
\noindent \emph{Contributions:} 

 \begin{itemize}[leftmargin=5mm]
 \item[(i)] 
    We present an analytical passivity condition for the FO-SLS model under short-memory discretization and show that this condition encompasses and generalizes previously proposed passivity formulations of viscoelastic virtual environments. We show that commonly used IO and FO viscoelastic models, such as the Kelvin-Voigt and Maxwell models, can be directly derived as particular cases of the proposed condition, by setting specific parameters to zero or infinity. Moreover, by setting the differentiation order to integer values, the formulation naturally reduces to IO passivity conditions, unifying the passivity results for FO and classical IO viscoelastic representations. Furthermore, since the derivation is performed under the short-memory discretization of Grünwald-Letnikov derivative, the results can be practically applied for real-time rendering with any memory window length. \smallskip
 \item[(ii)] 
    We systematically compare the perceived realism of various virtual FO environments across different memory truncation lengths in a human-subject experiment that evaluates haptic rendering realism based on participants' preferences. We provide experimental evidence of the existence of an upper bound on the memory length, after which the differences are not perceivable by the participants, justifying the use of truncated models. 
\end{itemize}

\begin{figure*}[t!]
  \includegraphics[width=.95\textwidth]{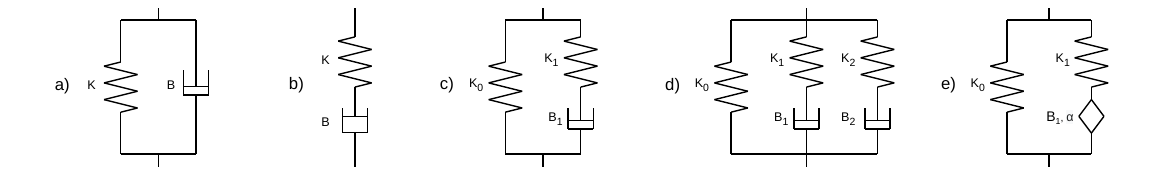}  
  \vspace{-1.0\baselineskip}
  \caption{Common viscoelastic models:
(a) Kelvin–Voigt model,
(b) Maxwell model,
(c) Standard linear solid model (Maxwell representation),
(d) Generalized Maxwell model with 5 parameters,
(e) Fractional-order standard linear solid model.}\label{fig:visco_models}\vspace{-3mm}
\end{figure*}

\vspace{-3mm}
\section{Related Work} \vspace{-1mm}

\subsection{Viscoelastic Models}
Precise capture of complex time-dependent behaviors of viscoelastic materials, such as creep and stress relaxation, is necessary for providing realistic haptic feedback. Where creep represents the material's deformation with time under a constant level of stress, stress relaxation is the decay in stress with time at a given constant level of strain. Several viscoelastic models, as depicted in Fig.~\ref{fig:visco_models}, have been proposed to model these phenomena, including the Maxwell and Kelvin-Voigt models~\cite{fractional_model_kelvin_max}. Unfortunately, most models cannot capture both behaviors with great accuracy. 

Classical viscoelastic models, including the Kelvin–Voigt, Maxwell, standard linear solid~(SLS), and generalized Maxwell models, have traditionally been utilized to approximate the time-dependent behavior of biological and synthetic materials. The Kelvin–Voigt model satisfactorily describes creep as it asymptotically approaches the strain value under constant stress; however, it is less useful in stress-relaxation situations due to its poor dynamic range~\cite{kelvin_orj,voigt_orj,kelvin}. Conversely, the Maxwell model effectively captures stress relaxation via an exponential decay, that is consistent with experimentally observed behaviors. On the other hand, it predicts unbounded linear creep under constant stress and thus is often inappropriate for many biological tissues~\cite{maxwell_orj,kelvin-maxwell}. The SLS (also called Zener) model represents a significant improvement over these limitations, as the combination of spring–dashpot elements enables this model to more appropriately represent both creep and stress relaxation behaviors; however, its single relaxation time limits its fidelity when modeling materials with broader relaxation spectra~\cite{zener_orj,frac_sls}. 

Arguably the most expressive family of models, the generalized Maxwell formulation overcomes these issues by including multiple Maxwell branches, each representing a different relaxation time, thereby allowing closer approximation of complex viscoelastic responses~\cite{wiechert_orj,frac-intg_model}. However, finite implementations of such models remain incomplete because they fail to fully capture the history dependence underlying viscoelasticity. Notably, in the limit, as the number of Maxwell branches approaches infinity, the generalized Maxwell model converges to a FO-SLS representation with a fractional damper that naturally captures long-range memory effects and more accurately represents viscoelastic materials~\cite{kelvin,gmm-sls}.

The FO-SLS model extends the classical IO-SLS model by the introduction of fractional derivatives and gives a more realistic representation of viscoelastic materials with complex behavior~\cite{frac_sls}. Fractional derivatives of FO dampers include history data from previous states; accordingly, FO calculus has been successfully utilized for modeling viscoelastic materials~\cite{cagatay_liver,tokatli2015stability,tokatli_patoglu_2018}. FO models have been shown to be effective in capturing intermediate viscoelastic responses between solid-like and fluid-like behaviors, for materials showing creep and relaxation characteristics~\cite{frac_better}. FO calculus has been identified as a better tool for describing such materials, since FO models require fewer parameters and simpler mathematical tools~\cite{fractals}. 

\vspace{-3mm}
\subsection{Passivity of Haptic Rendering}


Passivity theory gained prominence in haptic rendering, due to the seminal works of Colgate and Schenkel, who applied the passivity theorems to sampled-data systems to formulate passivity criteria for haptic rendering with discrete-time controllers~\cite{colgate}. The concept of passivity plays an essential role in haptic rendering, since passivity guarantees coupled stability without modeling the human operator. Furthermore, passivity conditions are analytical; hence, they can provide direct insights into the plant parameters and controller gains that affect the coupled stability of the interactions.

Weir simplified the proof in~\cite{colgate} and further extended the passivity condition, such that it is applicable to more general plants that possess damping that is non-decreasing with respect to frequency~\cite{Weir_thesis}. Haddadi and Hashtrudi-Zaad further extended the result to the velocity-sampled case~\cite{Haddadi2010}. Diolaiti~\etal included the effects of quantization, time delay, and Coulomb friction, through an energy-based analysis~\cite{Diolaiti2006}.  Finally, Fardad and Bamieh~\cite{Fardad2009} used frequency-response operators to describe the sampled-data system and established passivity results that encompass previous bounds as special cases within a more general framework.

The introduction of \emph{virtual coupling} between a haptic interface and a multibody simulation is a well-established means of ensuring coupled stability of interactions in a modular manner~\cite{Colgate1995,Brown1997,BrownPhD}. In particular, if the passivity of the sampled-data system comprising the haptic interface and the virtual coupler is guaranteed, then it is sufficient to ensure discrete-time passivity of the multibody simulation capturing the dynamics of the virtual environment to guarantee coupled stability. Given that virtual couplers typically take the form of viscoelastic models for impedance-type haptic interfaces, the passivity bounds established for rendering viscoelastic virtual environments also yield closely related results that may guide the design of virtual couplers~\cite{Colgate1995,Brown1997,BrownPhD, Adams1998, Adams1999}.   

To render realistic viscoelastic virtual environments, Son and Bhattacharjee applied passivity criteria to the SLS model for tissue rendering and derived analytical passivity bounds for the IO-SLS model~\cite{son2016passivity}. Tokatli and Patoglu studied the passivity of haptic rendering with FO system models and established passivity conditions for FO plants and/or virtual environments~\cite{tokatli2015stability,tokatli_tez,tokatli_patoglu_2018}. In this study, we further extend the basic results for FO models to establish passivity conditions for the FO-SLS model under the short-term memory discretization. Accordingly, we not only generalize the underlying model of viscoelasticity to one with higher representation power, but also rigorously address the effects of the finite memory requirements of haptic rendering with FO models. The results of this study provide a generalized framework for the previously studied viscoelasticity models.



\section{Preliminaries} \label{sec:Preliminaries}

\subsection{Fractional-Order Model of Viscoelasticity}

The FO-SLS model, illustrated in Fig.~\ref{fig:visco_models}.e, extends the classical SLS representation by incorporating a fractional order derivative  $D^\alpha$, where $\alpha\in (0,1)$, for the damping term. This allows for a more flexible modeling of viscoelastic behavior compared with classic integer-order models.

Several mathematically equivalent definitions of fractional derivatives exist. In this work, we adopt the Grünwald--Letnikov derivative~\cite{grunwald}, owing to its natural compatibility with discrete-time implementations. The Grünwald--Letnikov derivative of a function $f(x)$ is given by
\begin{equation}
D^\alpha f(x) = 
\lim_{h \to 0}
\frac{1}{h^\alpha}
\sum_{k=0}^{\infty} (-1)^k 
\binom{\alpha}{k}\,
f(x-kh),
\nonumber
\end{equation}
where $\alpha$ denotes the order of differentiation, $h$ is the step size, and the generalized binomial coefficient is defined as
\begin{equation}
\binom{\alpha}{k}
=
\frac{\Gamma(\alpha+1)}
{\Gamma(k+1)\Gamma(\alpha-k+1)},
\nonumber
\end{equation}
with $\Gamma(\cdot)$ representing the Gamma function.

Integer-order derivatives are recovered as special cases of fractional derivatives (e.g., $D^1$ is the first derivative and $D^0$ is the identity). In contrast to their integer-order counterparts, fractional derivatives are inherently nonlocal, indicating that the derivative at a given time depends on the entire history of the function. This property naturally captures memory-dependent effects such as creep and stress relaxation in viscoelastic media.

Fractional-order derivatives are linear operators and possess well-defined Laplace transforms. For the Grünwald--Letnikov fractional derivative of order~$\alpha$, the Laplace transform is
\begin{equation}
\mathcal{L}\{D^\alpha f(t)\}
=
s^\alpha F(s)
-
\sum_{k=0}^{n-1}
s^{\alpha-1-k}
f^{(k)}(0^+),
\nonumber
\end{equation}
where $n=\lceil \alpha \rceil$ denotes the smallest integer not less than~$\alpha$. Under the standard assumption used in transfer-function analysis that all initial conditions vanish, this expression reduces~to
\begin{equation}
\mathcal{L}\{D^\alpha f(t)\}
=
s^\alpha F(s),
\nonumber
\end{equation}
which serves as a direct generalization of the IO Laplace operator.

The force-position impedance transfer function of the FO-SLS model in Fig.~\ref{fig:visco_models}.e follows as
\begin{equation}
H(s)
=
K_0
+
\frac{
K_1 B_1 s^\alpha
}{
K_1
+
B_1 s^\alpha
},
\end{equation}
where $K_0$ and $K_1$ denote elastic stiffness coefficients, $B_1$ represents the coefficient of the fractional-order element, and the order $\alpha$ determines the ratio of elastic to viscous behavior for $\alpha \in (0,1)$.

\subsection{Short-Memory Discretization}

A practical challenge in implementing FO dynamics is that their exact computation requires infinite memory. Because this is infeasible, the convolution sum must be truncated to a finite memory depth. The short-memory approximation~\cite{short_memory} addresses this issue by exploiting the fact that the binomial weights in the Grünwald--Letnikov definition are largest for recent samples and decay rapidly for past data. Thus, a finite window of the most recent samples provides a sufficiently accurate approximation.

Under the short-memory approach, the $z$-transform of the fractional derivative of order $\alpha$ for a discrete-time signal $x(t)$ is approximated as
\begin{equation}
\mathcal{Z}\{D^\alpha x(t)\}
\approx
\frac{1}{T^\alpha}
\sum_{i=0}^{N}
c_i\,z^{-i},
\nonumber
\end{equation}
where $T$ is the sampling period, $N$ is the memory window length, and the coefficients $c_i$ are computed recursively as
\begin{align*}
c_{0} &= 1,\\
c_{i}
&=
(-1)^i \binom{\alpha}{i}
=
\frac{i-\alpha-1}{i}\,c_{i-1},
\qquad
i=1,2,\ldots,N.
\end{align*}

When $\alpha = 1$, the above recursion reduces to the standard backward-difference discretization, ensuring consistency with classical integer-order modeling.

Using this approximation, the discrete-time force--position impedance of the fractional-order SLS virtual environment becomes
\begin{equation}
\label{eq:SLS}
H(z)
=
K_0
+
\frac{
K_1 B_1 \frac{1}{T^{\alpha}} \sum_{i=0}^{N} c_i z^{-i}
}{
K_1
+
B_1 \frac{1}{T^{\alpha}} \sum_{i=0}^{N} c_i z^{-i}
},
\end{equation}

\noindent where the window length $N$ determines how accurately the discrete model captures viscoelastic phenomena such as stress relaxation and creep.

\section{Passivity of Fractional-Order SLS Model under Short-Memory Discretization}

The passivity of a discrete-time viscoelastic virtual environment can be assessed using the analytical condition introduced by Colgate~\etal~\cite{colgate} for the model in Fig.~\ref{fig:Sampled_data}. In particular, consider an haptic interface with mass $m$ and damping $b$, whose position $x$ is sampled at the rate $T$. The discrete-time virtual environment is represented by $H(z)$, while $ZOH$ denotes the zero-order hold. 
For this sampled-data system, a necessary and sufficient condition for passivity is given by
\begin{equation*}
\resizebox{\linewidth}{!}{$
\,\, b > 
\frac{T}{2\, (1 - \cos(\omega T))}
\, \Re\,\left\{
(1 - e^{-i \omega T})\, H(e^{i \omega T})
\right\}
\qquad
0 \le \omega \le \omega_N 
$}
\end{equation*}
where $\omega_N = \pi/T$ is the Nyquist frequency. 

\begin{figure}[h!]
    \centering
    \includegraphics[width=.85\linewidth]{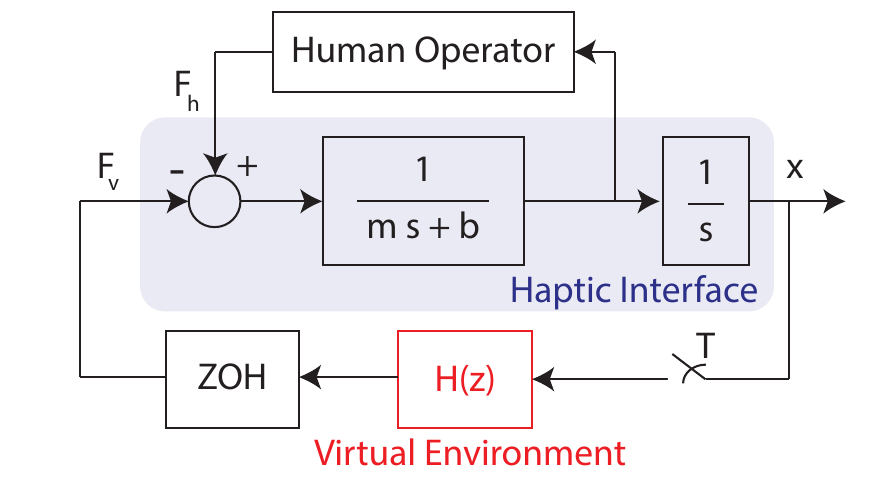}
    \vspace{-2mm}
    \caption{Model of the sampled-data system for haptic rendering.}
    \label{fig:Sampled_data} \vspace{-2mm}
\end{figure}

Substituting the discrete-time impedance of the FO-SLS model given in~ Eqn.~\eqref{eq:SLS} into $H(e^{i \omega T})$ yields a frequency-dependent passivity function for the right-hand side of the passivity condition. Define $f(\omega) = \frac{T}{2\big(1-\cos(\omega T)\big)} \Re\!\left\{\big(1-e^{-i \omega T}\big)\,H(e^{i \omega T})\right\}$. 

To derive a passivity condition in terms of the model parameters $K_0$, $K_1$, $B_1$, $\alpha$, $N$, and $T$, it is necessary to determine the maximum value of $f(\omega)$ over the interval $(0, \pi/T]$.

The following lemma establishes the frequencies where the passivity function is maximized for a finite memory length~$N$ and enables the subsequent derivation of the passivity bounds.

\begin{lemma}
\label{lem:passivity_n_finite}
Consider the discrete-time FO-SLS model, whose fractional derivative $s^{\alpha}$ is discretized using a short-memory expansion of length~$N$. Let $H(e^{i \omega T})$ denote the resulting discrete-time transfer function of the FO-SLS model, as given in~ Eqn.~\eqref{eq:SLS}.  For model parameters satisfying $K_1>0$, $B_1>0$, $0< \alpha <1$, and $T>0$, the passivity function $f(\omega)$ achieves its global maximum at the Nyquist frequency, when the short-term memory length~$N$ is \emph{odd}, that is,

\begin{equation*}
\label{eq:fomega_nyquist_max}
\max_{\omega \in (0,\pi/T]} f(\omega)
=
f\!\left(\frac{\pi}{T}\right),
\;\; \text{for odd $N$}.
\end{equation*}
\end{lemma}

\noindent A proof sketch is provided in the Supplementary Document~(SD)~\cite{Supplementary}. 

\begin{remark}[1]
If the short-memory length $N$ is \emph{even}, the maximum of the passivity function~$f(\omega)$ does not necessarily occur at the Nyquist frequency. Instead, there exists a frequency $\omega^\star < \pi/T$
such that
\begin{equation*}
\max_{\omega \in (0,\pi/T]} f(\omega)
=
f(\omega^\star), 
\;\; \text{for even $N$}.
\end{equation*}
\end{remark}

\begin{remark}[2]
\label{rem:passivity_infinite}
In the limit as the short-memory length approaches infinity, that is, $N \to \infty$, the passivity function $f(\omega)$ of the discrete-time FO-SLS model becomes strictly monotonically increasing over the interval $(0,\pi/T]$. Consequently, 
\begin{equation*}
\max_{\omega \in (0,\pi/T]} f(\omega) = f(\pi/T) \;\; \text{for $N \to \infty$}.
\end{equation*}
\end{remark}

To demonstrate these results, consider the numerical example in Fig.~\ref{fig:passivity_inset}, where the passivity function $f(\omega)$ vs  $\omega \, T$  is plotted for $N=100$, $N=101$, and $N \to \infty$, respectively. As expected, the asymptotic case exhibits strictly monotonic growth. In contrast, the truncated sums fluctuate around this asymptotic limit. For odd $N$ values, the maximum value remains at $\omega = \pi/T$ and slightly exceeds the asymptotic case, whereas for even $N$ values, the maximum shifts to a point $\omega^\star < \pi/T$, with the value at $\pi/T$ being lower than in the asymptotic case. Although this example is illustrative, the same trend holds across all cases, as discussed in the SD~\cite{Supplementary}.

\begin{figure}[h]
\centering

\setkeys{Gin}{trim=135 240 125 240,clip}

\begin{overpic}[width=.85\linewidth]{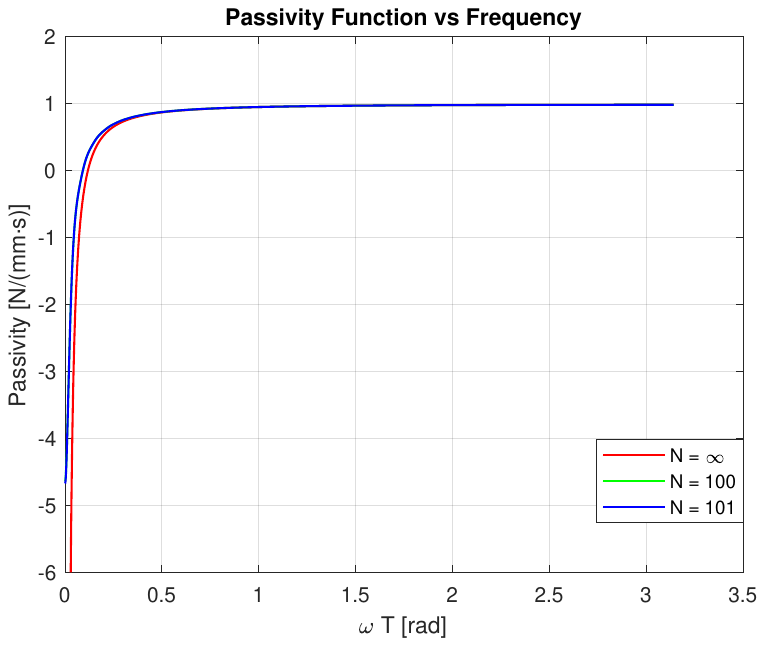}

\put(20,10.7){%
  \fcolorbox{red}{white}{%
   \setlength{\fboxsep}{6pt}   
    \setlength{\fboxrule}{1pt}
\includegraphics[
      width=0.625\linewidth,
      trim=110 253 125 260, 
      clip
    ]{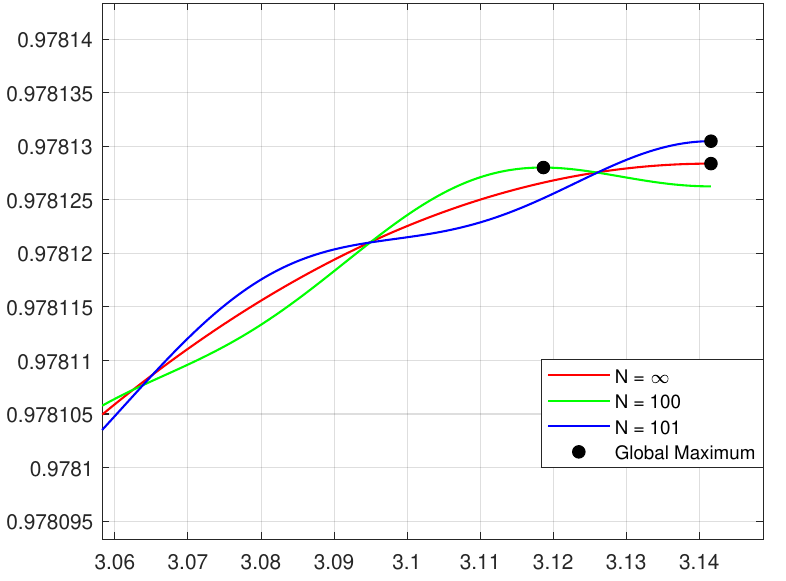}%
  }%
}

\color{red}
\linethickness{0.8pt}
\put(86,73){\oval(5,4)} 

\put(86,71){\vector(-4,-1){19}} 
\color{black}
\end{overpic}

\setkeys{Gin}{}
\vspace{-1.5mm}
\caption{Passivity function $f(\omega)$ versus $\omega\, T$ for $N = 100$, $N = 101$, and $N \to \infty$, for the parameters $K_0 = 0\,\mathrm{N/mm}$, $K_1 = 1\,\mathrm{N/mm}$, $B_1 = 1\,\mathrm{N \cdot s^{\alpha} / mm}$, $\alpha = 0.5$, and $T = 0.001\,\mathrm{s}$. The main plot shows the full frequency range, while the inset highlights a zoomed-in region around $\omega T = \pi$. } 
\label{fig:passivity_inset}\vspace{1mm}
\end{figure}

Given Lemma~\ref{lem:passivity_n_finite}, which establishes that the maximum of the passivity function occurs at the Nyquist frequency for odd memory length $N$, we now derive a closed-form passivity condition in terms of the model parameters by evaluating the function at $\omega = \pi/T$.


\begin{proposition}
\label{prop:passivity_nyquist}
Consider the discrete-time FO-SLS virtual environment model, whose fractional derivative $s^{\alpha}$ term is discretized using a short-memory expansion of length~$N$.   Let $H(e^{i \omega T})$ denote the resulting discrete-time transfer function of the FO-SLS model, as given in~ Eqn.~\eqref{eq:SLS}.  For model parameters satisfying $K_1>0$, $B_1>0$, $0< \alpha <1$, $T>0$,  and for odd memory length $N$, a necessary and sufficient condition for passivity of the FO-SLS model is given as

\begin{equation*}
b >
\frac{K_0 T}{2}
+
\frac{K_1 T}{2}
\frac{B_1 \Delta_p}{B_1 \Delta_p + K_1 T^{\alpha}}
\end{equation*}
where  \vspace{-2mm}
\begin{equation*}
\Delta_p
=
\sum_{i=0}^{N} \, (-1)^i \, c_i,
\qquad
c_i = \binom{\alpha}{i} \, (-1)^i
\end{equation*}
are the short-memory discretization coefficients of the Grünwald–Letnikov derivative. This sum can also be expressed using the hypergeometric function as
$$ \Delta_p = \sum_{i=0}^{N} \binom{\alpha}{i} = \, {}_2F_1(-N, -\alpha; 1; -1)$$
\noindent where ${}_2F_1$ is the Gauss hypergeometric function.

\end{proposition}

\noindent A proof sketch is provided in the SD~\cite{Supplementary}.

\begin{corollary} 
\label{rem:passivity_infinite_n}
In the asymptotic case as $N \to \infty$, the passivity condition in the Proposition~\ref{prop:passivity_nyquist} reduces to
\begin{equation*}
b >
\frac{K_0 T}{2}
+
\frac{K_1 T}{2}
\cdot
\frac{B_1 2^{\alpha}}{B_1 2^{\alpha} + K_1 T^{\alpha}}.
\end{equation*}
\end{corollary}

\begin{proof}
As $N \to \infty$, the alternating sum converges to $$\Delta_{p,\infty} = \sum_{i=0}^{\infty} \binom{\alpha}{i} = 2^{\alpha}, \qquad 0 < \alpha < 1.$$ The results follows from 
Proposition~\ref{prop:passivity_nyquist}.
\end{proof}

\begin{corollary} 
\label{rem:passivity_truncated}
For a finite and odd memory length $N$, a sufficient condition for the passivity of the discrete-time transfer function of the FO-SLS model in~ Eqn.~\eqref{eq:SLS}
is given as
\begin{equation*}
b >
\frac{K_0 T}{2}
+
\frac{K_1 T}{2}
\frac{B_1 \Delta_p}{B_1 \Delta_p + K_1 T^{\alpha}}
\end{equation*}
where  \vspace{-2mm}
\begin{equation*}
\Delta_p = 2^{\alpha} - \binom{\alpha}{N+1}.
\end{equation*}
\end{corollary}

\begin{proof} 
For a finite and odd memory length $N$ and for $ 0 < \alpha < 1$, the alternating nature of the series implies that
$$\sum_{i=0}^{N} \binom{\alpha}{i} + \binom{\alpha}{N+1} < \sum_{i=0}^{\infty} \binom{\alpha}{i}.$$
Hence, the result follows from  Proposition~\ref{prop:passivity_nyquist}, by using the following conservative bound for the truncated sum, 
$$\sum_{i=0}^{N} \binom{\alpha}{i} < 2^\alpha - \binom{\alpha}{N+1}.\vspace{-6mm}$$ 
\end{proof}


\begin{remark}[3] The proposed FO-SLS formulation provides a unified formulation that reduces to classical integer-order~(IO) and fractional-order~(FO)  Kelvin-Voigt, Maxwell, and SLS models under appropriate parameter selections and limits. Table~\ref{tab:passivity_special_cases} summarizes the corresponding reductions and the resulting passivity conditions.
\end{remark}

\begin{table}[h!]
\centering
\caption{Generalization of passivity bounds to existing viscoelastic models.}
\vspace{-1mm}
\label{tab:passivity_special_cases}
\renewcommand{\arraystretch}{2.25} 
\setlength{\tabcolsep}{3pt}      
\footnotesize                     
\begin{tabularx}{\columnwidth}{l l X}
\hline
\textbf{Model} & \textbf{Parameters} & \textbf{Passivity bound} \\
\hline 

FO-SLS  & -- &
$\displaystyle b > \frac{K_0 T}{2} + \frac{K_1 T}{2} \frac{B_1 \Delta_p}{B_1 \Delta_p + K_1 T^{\alpha}}$ \\

FO-Kelvin-Voigt & $K_1 \to \infty$ &
$\displaystyle b > \frac{K_0 T}{2} + \frac{B_1 \Delta_p}{2 T^{\alpha-1}}$\\

FO-Maxwell & $K_0 = 0$ &
$\displaystyle b > \frac{K_1 T}{2} \frac{B_1 \Delta_p}{B_1 \Delta_p + K_1 T^{\alpha}}$ \\

IO-SLS & $\alpha = 1$ &
$\displaystyle b > \frac{K_0 T}{2} + \frac{K_1 B_1 T}{2 B_1 + K_1 T}$ \\

IO-Kelvin-Voigt & $K_1 \to \infty, \alpha = 1$ &
$\displaystyle b > \frac{K_0 T}{2} + B_1$ \\

IO-Maxwell & $K_0 = 0, \alpha = 1$ &
$\displaystyle b > \frac{K_1 B_1 T}{2 B_1 + K_1 T}$ \\

\hline
\end{tabularx}
\end{table}

\vspace{-3mm}
\section{Effective Impedance of FO-SLS Model}

Insight into the rendered dynamics can be obtained by studying the effective impedance of the closed-loop system~\cite{mehling2005increasing, colonnese2015rendered}. Effective impedance definitions decompose the frequency-dependent impedance function into its fundamental components, where the real positive part is associated with the effective damping, while the imaginary part is assigned to the effective spring and effective inertia based on their phase characteristics. For viscoelastic rendering, we focus on the effective spring and effective damping of the FO-SLS model.

\subsection{Effective Stiffness}

Effective stiffness is useful for understanding the net stiffness behavior of a complex virtual environment that may contain multiple springs, dampers, and fractional-order elements. Given that the fractional-order element exhibits a complex response lying between elasticity and viscosity, the effective stiffness definition allows for the extraction of the frequency-dependent stiffness from the overall viscoelastic response.

The effective stiffness of a discrete-time virtual environment with impedance $H(e^{i \omega T})$ is defined as~\cite{effectiveKB}
\begin{equation*}
    ES(\omega) = \Re^{+}\!\big\{ H(e^{i \omega T}) \big\}
\end{equation*}
\noindent where $\Re^{+}\{\cdot\}$ denotes the positive real part. 

\begin{proposition}
\label{prop:ES}
Consider the discrete-time  FO-SLS model-based virtual viscoelastic environment as in Eqn.~\eqref{eq:SLS}. The frequency-dependent effective stiffness of the model is given as 
\begin{equation*}
ES(\omega)
=
K_0
+
\Re\left\{
\frac{
B_1 K_1 \displaystyle S(\omega)
}{
K_1 T^\alpha + B_1 \displaystyle S(\omega)
}
\right\},
\end{equation*}
Equivalently,
\begin{equation*}
ES(\omega)
=
K_0
+
\frac{
B_1 K_1 \,\Re\left\{
\left(K_1 T^\alpha + B_1 S(\omega)\right) S^*(\omega)
\right\}
}{
\left|K_1 T^\alpha + B_1 S(\omega)\right|^2
}.
\end{equation*}
\noindent with $S(\omega)$ and its complex conjugate $S^{*}(\omega)$ are defined as
\begin{equation*}
S(\omega)
=
\sum_{k=0}^{N} c_k e^{ik\omega T},
\qquad
S^*(\omega)
=
\sum_{k=0}^{N} c_k e^{-ik\omega T}.
\end{equation*}
\noindent where $c_k$ are the short-memory discretization coefficients of the Grünwald–Letnikov derivative.
\end{proposition}

\noindent A proof sketch is provided in the SD~\cite{Supplementary}.

\begin{corollary}
\label{cor:ES_infmemory}
In the asymptotic case as $N \to \infty$, the  frequency dependent effective stiffness of the model in Proposition~\ref{prop:ES} reduces to
\begin{equation*}
\resizebox{\columnwidth}{!}{$
\begin{array}{l}
    ES(\omega) = K_0 \, +\\[6pt]
    \dfrac{
         K_1 B_1\Big(
            B_1 4^{\alpha} \sin^{2\alpha}\!\big(\tfrac{\omega T}{2}\big)
            +
            K_1 T^{\alpha} 2^{\alpha} \sin^{\alpha}\!\big(\tfrac{\omega T}{2}\big)
            \cos\!\big(\tfrac{(\omega T - \pi)\alpha}{2}\big)
        \Big)
    }{
        K_1^{2} T^{2\alpha}
        +
        B_1^{2} 4^{\alpha} \sin^{2\alpha}\!\big(\tfrac{\omega T}{2}\big)
        +
        2 B_1 K_1 T^{\alpha} 2^{\alpha} \sin^{\alpha}\!\big(\tfrac{\omega T}{2}\big)
        \cos\!\big(\tfrac{(\omega T - \pi)\alpha}{2}\big)
    }.
\end{array}
$}
\end{equation*}\label{asym_K}
\end{corollary}

\noindent A proof sketch is provided in the SD~\cite{Supplementary}.


\begin{remark}[4]
    Using complex terms, the effective stiffness formula can also be recast more compactly as follows: \vspace{1mm}
\begin{equation*}
ES(\omega)
=K_0
+\,\Re\!\left\{
\frac{K_1 B_1\,(1-e^{-i \omega T})^{\alpha}}
     {K_1T^{\alpha}+B_1\,(1-e^{-i \omega T})^{\alpha}}
\right\}.
\end{equation*}
\end{remark}


\begin{corollary}
\label{cor:ES_lowfreq}
    The effective stiffness of the model in Proposition~\ref{prop:ES} at low frequencies, that is, around $\omega = 0$, for an arbitrary finite number of memory terms~$N$, is given by
\begin{equation*}
    ES = K_0 + \frac{ K_1 B_1\,\Delta_s}{B_1 \Delta_s + K_1 T^{\alpha}},
\end{equation*}
\textit{where} \vspace{-3mm}
\begin{equation*}
    \Delta_s = \binom{N - \alpha}{N}
\end{equation*}
is the generalized binomial term. \label{low_freq_K}
\end{corollary}
\noindent A proof sketch is provided in the SD~\cite{Supplementary}.

Corollary~\ref{cor:ES_infmemory} focuses on the effective stiffness of the asymptotic case as $N \to \infty$, while Corollary~\ref{cor:ES_lowfreq} captures the finite memory case around low frequencies, as the series expansion is very complex for the most general case. On the other hand, the low-frequency expression is quite useful in practice, since human excitations typically lie well below the Nyquist frequency, within the $0$–$10$~Hz band. Accordingly, for many applications, it may be sufficient to evaluate the effective stiffness around $\omega = 0$ to characterize the perceived stiffness behavior. Furthermore, the asymptotic case is recovered as $N$ is increased.

The effective stiffness plots investigating the influence of different memory lengths $N$ and fractional orders $\alpha$ of the transfer function $H(z)$ in Eqn.~\eqref{eq:SLS} are provided in the SD~\cite{Supplementary}.

\subsection{Effective Damping}

Effective damping complements effective stiffness to describe the overall response of a viscoelastic virtual environment, capturing the damping contributions of the FO elements. 

The effective damping of a discrete-time virtual environment with impedance $H(e^{i \omega T})$ is defined as~\cite{effectiveKB}
\begin{equation*}
    ED(\omega)
    =
    \frac{1}{\omega}
    \Im^{+}\!\big\{ H(e^{i \omega T}) \big\},
\end{equation*}
where $\Im^{+}\{\cdot\}$ denotes the positive imaginary part.

\begin{table*}[t!]
\centering
\caption{Generalization of effective stiffness \(ES(\omega)\) and damping \(ED(\omega)\) to existing viscoelastic models as special cases.}
\vspace{-1mm}
\label{tab:ES_special_cases}
\renewcommand{\arraystretch}{2.5}
\resizebox{.85\linewidth}{!}{
\begin{tabular}{c c c c}  
\hline
\textbf{Model} & \textbf{Parameter selection} & \textbf{Effective stiffness \(ES(\omega)\)} & \textbf{Effective damping \(ED(\omega)\)} \\
\hline

FO SLS  & -- &
\(\displaystyle
K_0
+\,\Re\!\left\{
\frac{K_1B_1\,(1-e^{-i \omega T})^{\alpha}}
     {K_1T^{\alpha}+B_1\,(1-e^{-i \omega T})^{\alpha}}
\right\} \) & \(\displaystyle
\frac{1}{\omega}\,
\Im\!\left\{
\frac{
K_1B_1(1-e^{-i \omega T})^{\alpha}
}{
K_1T^{\alpha}+B_1(1-e^{-i \omega T})^{\alpha}
}
\right\}
\) \\

FO Kelvin--Voigt & \(K_1 \to \infty\) &
\(\displaystyle
K_0+\frac{B_1}{T^{\alpha}}\,
\Re\!\left\{(1-e^{-i \omega T})^{\alpha}\right\} \) & \(\displaystyle
\frac{B_1}{\omega\,T^{\alpha}}\,
\Im\!\left\{
(1-e^{-i \omega T})^{\alpha}
\right\}
\) \\

FO Maxwell & \(K_0 = 0\) &
\(\displaystyle
\,\Re\!\left\{
\frac{K_1B_1\,(1-e^{-i \omega T})^{\alpha}}
     {K_1T^{\alpha}+B_1\,(1-e^{-i \omega T})^{\alpha}}
\right\} \) & \(\displaystyle
\frac{1}{\omega}\,
\Im\!\left\{
\frac{
K_1B_1(1-e^{-i \omega T})^{\alpha}
}{
K_1T^{\alpha}+B_1(1-e^{-i \omega T})^{\alpha}
}
\right\}
\) \\

IO SLS & \(\alpha = 1\) &
\(\displaystyle
K_0
+\,\Re\!\left\{
\frac{K_1B_1\,(1-e^{-i \omega T})}
     {K_1T+B_1\,(1-e^{-i \omega T})}
\right\} \) & \(\displaystyle
\frac{1}{\omega}\,
\Im\!\left\{
\frac{K_1B_1\,(1-e^{-i \omega T})}
{K_1T+B_1\,(1-e^{-i \omega T})}
\right\}\) \\

IO Kelvin--Voigt & \(K_1 \to \infty,\;\alpha=1\) &
\(\displaystyle
K_0+\frac{B_1}{T}\big(1-\cos(\omega T)\big) \) & \(\displaystyle
\frac{B_1\sin(\omega T)}{\omega T}
\) \\

IO Maxwell & \(K_0 = 0,\;\alpha=1\) &
\(\displaystyle
\,\Re\!\left\{
\frac{K_1B_1\,(1-e^{-i \omega T})}
     {K_1T+B_1\,(1-e^{-i \omega T})}
\right\} \) & \(\displaystyle
\frac{1}{\omega}\,
\Im\!\left\{
\frac{K_1B_1\,(1-e^{-i \omega T})}
{K_1T+B_1\,(1-e^{-i \omega T})}
\right\}
\) \\

\hline \normalsize
\end{tabular} }\vspace{-7mm} 
\end{table*}

\begin{proposition}
\label{prop:ED}
Consider the discrete-time FO-SLS model-based virtual viscoelastic environment as in Eqn.~\eqref{eq:SLS}. The frequency-dependent effective damping of the model is given as 
\begin{equation*}
ED(\omega)
=
\frac{1}{\omega}
\Im\left\{
\frac{
B_1K_1 S(\omega)
}{
K_1T^\alpha+B_1S(\omega)
}
\right\},
\end{equation*}
Equivalently,
\begin{equation*}
ED(\omega)
=
\frac{1}{\omega}
\frac{
B_1K_1
\Im\left\{
\left(K_1T^\alpha+B_1S(\omega)\right)S^*(\omega)
\right\}
}{
\left|K_1T^\alpha+B_1S(\omega)\right|^2
}.
\end{equation*}
\noindent with $S(\omega)$ and its complex conjugate $S^{*}(\omega)$ are defined as
\begin{equation*}
S(\omega)
=
\sum_{k=0}^{N} c_k e^{ik\omega T},
\qquad
S^*(\omega)
=
\sum_{k=0}^{N} c_k e^{-ik\omega T}.
\end{equation*}
\noindent where $c_k$ denote the short-memory discretization coefficients of the Grünwald–Letnikov derivative.
\end{proposition}

\noindent A proof sketch is provided in the SD~\cite{Supplementary}.

\begin{corollary}
\label{cor:ED_infmemory}
In the asymptotic case as $N \to \infty$, the frequency dependent effective damping of the model in Proposition~\ref{prop:ED} reduces to
\begin{equation*}
\resizebox{\columnwidth}{!}{$
\begin{array}{l}
    ED(\omega) = \\[6pt]
    \quad
    \dfrac{
        - B_1 K_1^{2} T^{a} 2^{a}
        \sin^{a}\!\big(\tfrac{\omega T}{2}\big)
        \sin\!\big(\tfrac{(\omega T - \pi)a}{2}\big)
    }{
        \omega \Big(
            K_1^{2} T^{2a}
            +
            B_1^{2} 4^{a} \sin^{2a}\!\big(\tfrac{\omega T}{2}\big)
            +
            2 B_1 K_1 T^{a} 2^{a} \sin^{a}\!\big(\tfrac{\omega T}{2}\big)
            \cos\!\big(\tfrac{(\omega T - \pi)a}{2}\big)
        \Big)
    }.
\end{array}
$}
\end{equation*}\label{asym_B}
\noindent A proof sketch is provided in the SD~\cite{Supplementary}.
\end{corollary}

\begin{remark}[5]
     Using complex terms, the effective damping formula can also be recast more compactly as follows: \vspace{1mm}
\begin{equation*}
ED(\omega)
=
\frac{1}{\omega}\,
\Im\!\left\{
\frac{K_1
B_1\big(1-e^{-i \omega T}\big)^{\alpha}
}{
K_1T^{\alpha}
+
B_1\big(1-e^{-i \omega T}\big)^{\alpha}
}
\right\}.
\end{equation*}
\end{remark}

\smallskip
\begin{corollary}
\label{cor:ED_lowfreq}
    The effective damping of the model in Proposition~\ref{prop:ED} at low frequencies, that is, around $\omega = 0$, for an arbitrary number of memory terms $N$ is given by
\begin{equation*}
    ED =     \frac{B_1 K_1^{2} T^{\alpha+1} \,\Delta_d}{\big( B_1 \Delta_s + K_1 T^{\alpha} \big)^{2}},
\end{equation*}
\textit{where} \vspace{-3mm}
\begin{equation*}
    \Delta_d = \alpha \, \binom{N - \alpha}{N - 1},
    \qquad
    \Delta_s = \binom{N - \alpha}{N}.
\end{equation*}
are the generalized binomial terms. \label{low_freq_B}

\end{corollary}
\noindent A proof sketch is provided in the SD~\cite{Supplementary}.

As in the effective stiffness case, Corollary~\ref{cor:ED_infmemory} focuses on the effective damping of the asymptotic case as $N \to \infty$, while Corollary~\ref{cor:ED_lowfreq} captures the finite memory case around low frequencies, as the series expansion is very complex for the most general case. On the other hand, for many applications, it may be sufficient to evaluate the effective damping around $\omega = 0$ to characterize the perceived stiffness behavior. As expected, the asymptotic case is recovered as $N$ is increased. 


The effective damping plots investigating the influence of different memory lengths $N$ and fractional orders $\alpha$ of the transfer function $H(z)$ in Eqn.~\eqref{eq:SLS} are provided in the SD~\cite{Supplementary}.



\begin{remark}[7] The proposed FO-SLS effective impedance formulations provide a unified formulation that reduces to classical integer-order~(IO) and fractional-order~(FO)  Kelvin-Voigt, Maxwell, and SLS models under appropriate parameter selections and limits. Table~\ref{tab:ES_special_cases}  summarizes the corresponding reductions and the resulting effective stiffness and effective damping expressions.
\end{remark}

\vspace{-3mm}
\subsection{Interpretation of Effective Stiffness and Damping}

The expressions for the effective stiffness \(ES(\omega)\) and effective damping \(ED(\omega)\) admit a mechanical interpretation in terms of an equivalent frequency-dependent impedance network. Define the fractional-order viscoelastic element as
\begin{equation*}
B_{\mathrm{fo}}(\omega)
\triangleq
\frac{B_1}{T^{\alpha}}\,
\big(1-e^{-i \omega T}\big)^{\alpha}.
\end{equation*}

The term $(1-e^{-i \omega T})^{\alpha}/T^{\alpha}$ arises from the discrete-time realization of the fractional derivative $s^{\alpha}$ according to the Grünwald--Letnikov definition, where the corresponding binomial theorem is employed, and the memory length~$N$ is taken to infinity.
Moreover, in the low-frequency limit, the resulting frequency response converges to that of the continuous-time operator $s^{\alpha}$, while it remains bounded at high frequencies due to discretization, which is essential for discrete-time performance and passivity analysis. In particular,
\begin{equation*}
\label{eq:GL_lowfreq_limit}
\lim_{\omega \to 0}
\frac{\big(1-e^{-i \omega T}\big)^{\alpha}}{T^{\alpha}}
=
(i \omega)^{\alpha},
\end{equation*}
\noindent which shows that the discrete-time operator approaches the continuous-time fractional derivative behavior as $\omega  \, T \to 0$.

Considering the $B_{\mathrm{fo}}$ definition, the terms appearing in both $ES(\omega)$ and $ED(\omega)$ correspond to the equivalent impedance of a \emph{series connection} between a spring with stiffness $K_1$ and a frequency-dependent complex term $B_{\mathrm{fo}}(\omega)$, placed in parallel with the static stiffness $K_0$. The real part of this series combination contributes to elastic energy storage and is captured by $ES(\omega)$, while the imaginary part, scaled by $1/\omega$, quantifies energy dissipation and is captured by $ED(\omega)$.

The magnitude and phase of $B_{\mathrm{fo}}(\omega)$ are jointly governed by the sampling time $T$ and the fractional order $\alpha$. 
In particular, decreasing $T$ increases the high-frequency magnitude of $B_{\mathrm{fo}}(\omega)$ proportionally to $T^{-\alpha}$, making the series branch effectively stiffer near the Nyquist frequency.  Meanwhile, the fractional order $\alpha$ controls both the magnitude growth, proportional to $\sin^{\alpha}(\omega T/2)$, and the phase rotation of $B_{\mathrm{fo}}(\omega)$, thereby shaping the trade-off between energy storage and dissipation across frequencies.

\section{Validations of Coupled Stability}
\label{sec:passivity_experiment}

In this section, we provide experimental evidence to help validate the theoretical passivity bounds. 
\subsection{Identification of Plant Parameters}

The haptic interface consists of a Bei Kimco LA28-15-002Z linear actuator equipped with a 2000 counts/inch U.S.~Digital linear optical encoder and a 3D printed supporting structure, as shown in Fig.~\ref{fig:Haptic Device}. 
 Linear guides are used to support non-collinear forces and minimize friction during interaction with the handle.

The direct-drive impedance-type plant dynamics can be faithfully modeled as a first-order system consisting of an equivalent mass and viscous damping, with an impedance transfer function of $Z_{plant}(s) = {m \, s + b}.$

The plant parameters were determined by conducting system identification experiments, during which a bidirectional linear chirp force signal, sweeping from 1~Hz to 10~Hz over 15~s, was commanded in real-time at 1~kHz using the Matlab Real-Time environment, while the resulting motions were recorded. A first-order transfer function was fitted to the data to determine the plant parameters, yielding $m = 73.4$~g and $b = 0.0025$~N/(mm·s), with $R^2 = 0.90$. 

\begin{figure*}[t]
    \centering
    \includegraphics[width=0.35\textwidth]{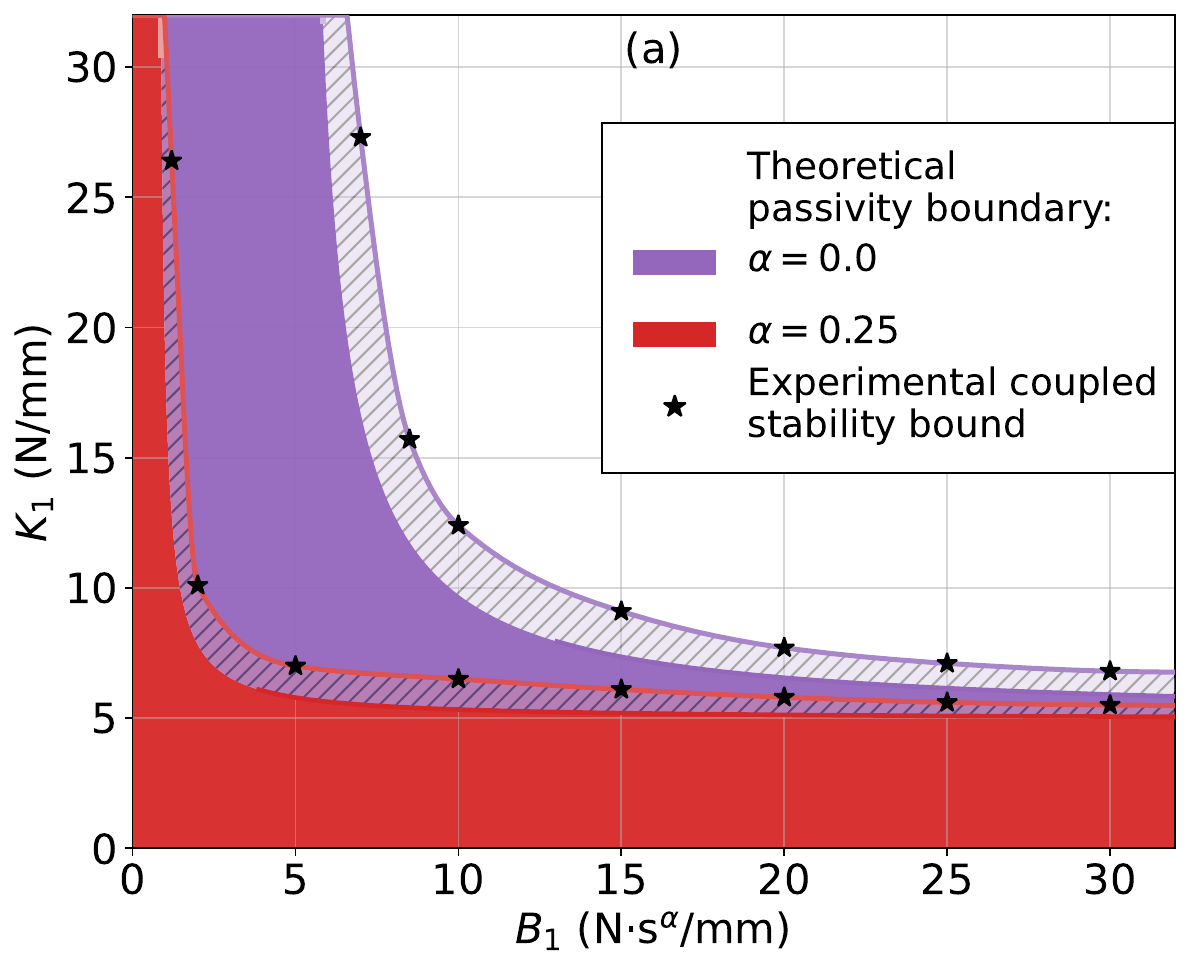} 
     \hspace{15mm}
    \includegraphics[width=0.35\textwidth]{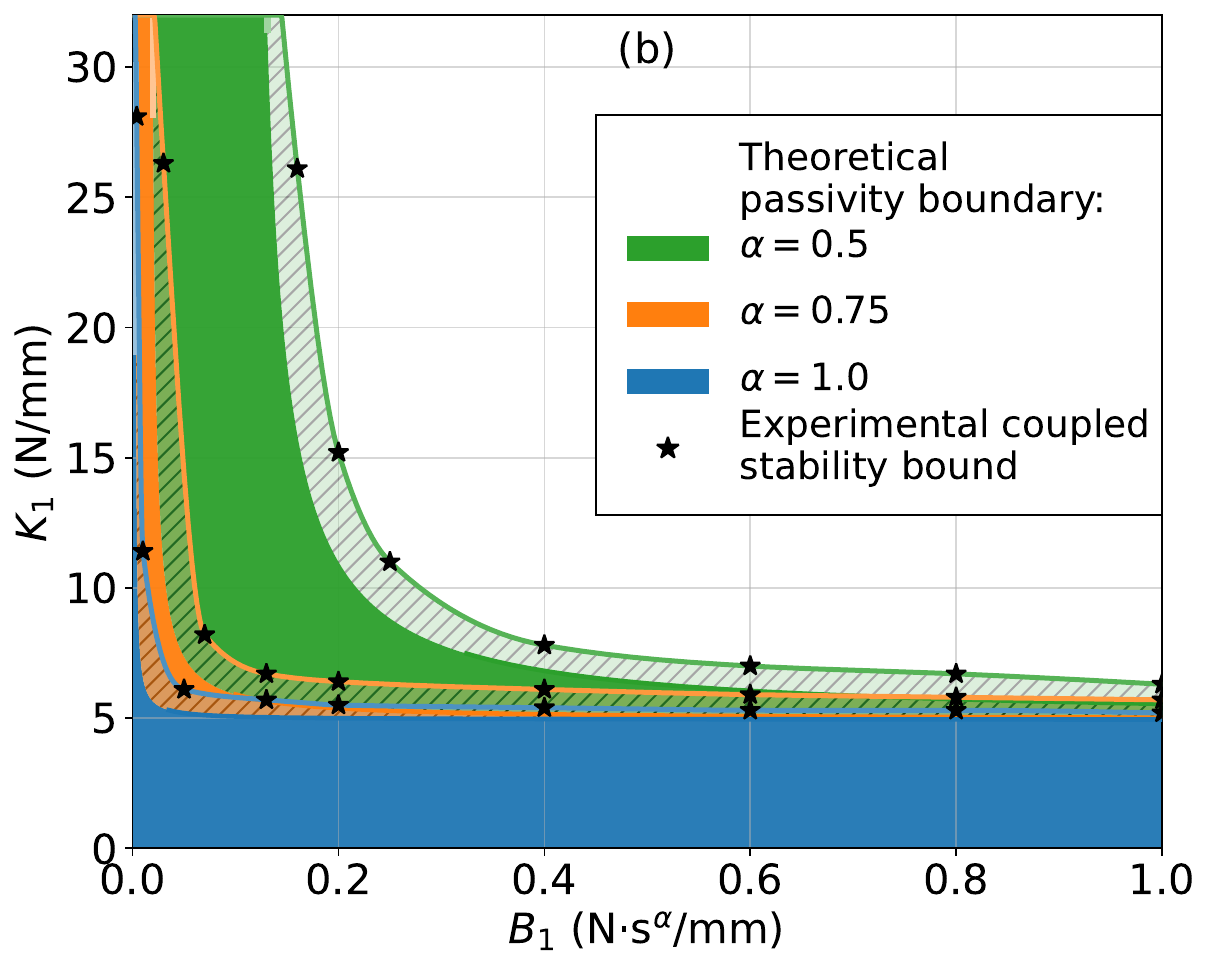} 
    \vspace{-.5\baselineskip}
    \caption{The solid colors present analytical passivity regions of the FO-SLS model in the $(B_1, K_1)$ plane for different orders $\alpha$, where $K_0 = 0$ to isolate the effect of the fractional viscoelastic component. The left figure presents $\alpha \in \{0, 0.25\}$, while the right figure shows $\alpha \in \{0.5, 0.75, 1\}$, since the differences can be better demonstrated using different $B_1$ ranges. The experimentally determined parameters are marked with $\ast$ and shown alongside their corresponding curve-fitted boundaries. The shaded regions highlight the mismatch between the analytical passivity and the experimental coupled stability bounds.}
    \label{fig:passivity_alpha_regions} \vspace{-.7\baselineskip}
\end{figure*}

\subsection{Experimental Verifications of Passivity Bounds}

The interaction controller was based on open-loop impedance control, implemented in real time at 1~kHz, with the reference impedance set to the FO-SLS model with $N=101$.

To experimentally determine generalized (with respect to~$\alpha$) $K_1-B_1$ plots for the FO-SLS model, a line search was conducted along the $K_1$-axis, starting from 25\% below the theoretical passivity bound and increasing the $K_1$ parameter with a resolution of $0.1~\mathrm{N/mm}$. $K_0$ was kept constant at $K_0 = 0$ to isolate the effect of the fractional viscoelastic component on passivity. Coupled stability of interactions was evaluated by exciting the system with an impact to excite all possible frequencies and observing the response. For each set of trial parameters, if no violation of coupled stability was observed across five trials, we concluded that the experimental evidence indicated the system was passive for that parameter set. Otherwise, if any violations of coupled stability (e.g., chatter) were observed, then the parameter set was considered active. A video of the coupled stability experiments is provided in the Multimedia Extension. 

The theoretical passivity regions in the $(B_1, K_1)$ plane for various fractional orders $\alpha$ are presented in Fig.~\ref{fig:passivity_alpha_regions}. The plots include not only the analytical boundaries of the passivity regions but also the experimentally determined parameters and the passivity-bound estimates based on these parameters. 

The experimentally determined parameters agree closely with the analytical predictions, with an NRMSE of $1.61\%$, where errors are computed based on the shortest distance to the theoretical boundary. 
 The experimental results based on the coupled stability of interactions slightly overestimate the passivity bounds, extending the passivity region by~$\approx 15\%$. 

These results indicate a reasonable agreement between the analytical passivity conditions and the experimentally observed coupled stability limits. The experimental regions tend to be slightly less conservative, generally expanding the passivity region compared to the analytical prediction. This discrepancy can be attributed to unmodeled physical effects such as friction and hysteresis, which introduce additional energy dissipation and are not captured by the underlying LTI system model.

\begin{figure*}[t]
    \centering \vspace{-.3\baselineskip}
    \includegraphics[width=0.8\textwidth]{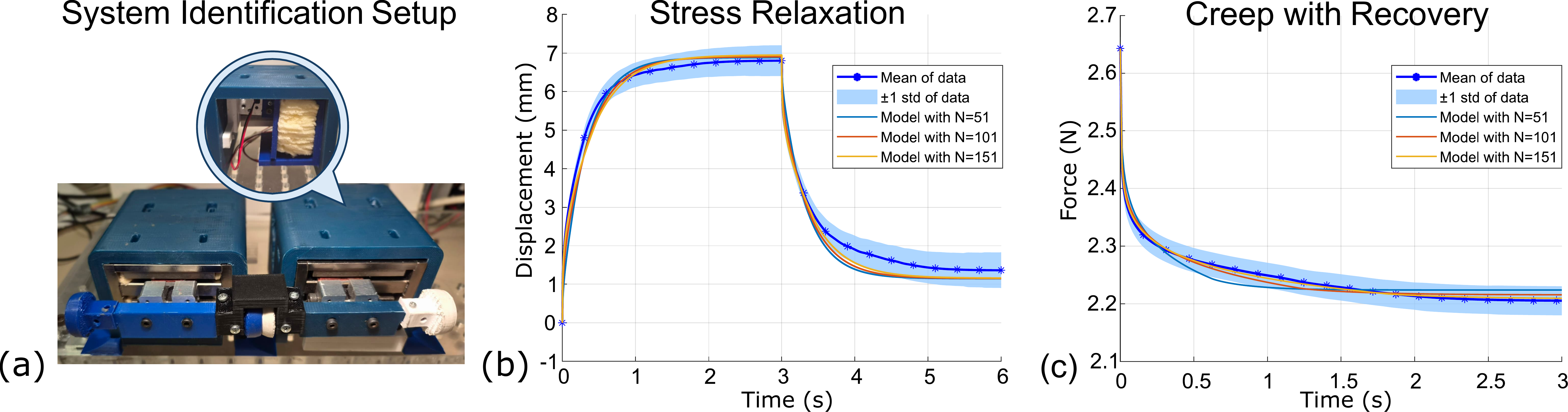}
    \caption{(a) System identification setup used to determine reference viscoelastic material. Measured responses for (b) stress relaxation under constant deformation and (c) creep--recovery under step loading. The corresponding fits of the fractional-order standard linear solid (SLS) model for different memory lengths are also included.}
    \label{fig:systemID} \vspace{-2mm}
\end{figure*}

\vspace{-2mm}
\section{Validations of Haptic Rendering Fidelity}
\label{sec:comparison_experiment}

In this section, we evaluate the rendering fidelity of the FO-SLS model through a human-subject study comparing virtual viscoelastic renderings with a physical viscoelastic material. We test memory lengths of $N \in \{51, 101, 151\}$ to investigate how the discretization window affects perceived realism.


\vspace{-3mm}
\subsection{Determination of Parameters of the FO-SLS Model}
\label{sec:population_optima}


For these human-subject studies, we first determine the parameters of the viscoelastic material using the four-parameter FO-SLS model in Eqn.~\eqref{eq:SLS}. 


The reference material used for all perceptual tests was a viscoelastic memory foam sample. We performed system identification of this sample using stress-relaxation and creep-with-recovery experiments. During these experiments, the material was characterized using the system identification setup shown in Fig.~\ref{fig:systemID}(a), in which the haptic interfaces were mechanically coupled to apply controlled excitations and record the resulting material response. Parameter estimation involved searching the feasible region of the parameter space that satisfies the analytically derived passivity criteria.

The system identification process involved applying a constant force of 3~N to the sample and observing its deformation behavior for 3~s. Thereafter, the force level was instantaneously reduced to 0.5 N, and the deformation behavior during recovery was measured. Stress relaxation experiments involved applying a constant deformation of 5~mm to the sample and measuring its response in terms of the force output over 3~s. All experiments were repeated 16~times. Figs.~\ref{fig:systemID}(b) and~\ref{fig:systemID}(c) show the average responses together with their standard deviations and the regression results.

The choice of the memory window length $N$ is of primary importance for the correct reproduction of the creep and stress relaxation behavior in the discrete-time FO-SLS model. For the model fitting, window lengths of $N \in \{51, 101, 151\}$ were chosen. The results are provided in Fig.~\ref{fig:systemID}. The fit quality is evaluated using the normalized RMSE~(NRMSE).

As shown in Figs.~\ref{fig:systemID}(b) and~\ref{fig:systemID}(c), increasing the window length improves model accuracy. For $N = 101$, an acceptable level of accuracy is achieved according to the NRMSE. Namely, the NRMSE equals to $0.324\%$ for the stress relaxation test. Hence, the window length $N = 101$ was selected as the nominal memory length because it provides high accuracy without requiring excessive computational effort.

Table~\ref{tab:hil_sysid_all_models} presents the identified model parameters for the FO-SLS models with  $N \in \{51, 101, 151\}$  and the accuracy of these fits. The identified model parameters have $K_0 < 0$, since this element compensates for the additional stiffness introduced by the fractional-order element. It is important to note that the identified parameters comply with the passivity condition, indicating a physically consistent and realistic representation of the viscoelastic material. Another important observation is that models with different memory lengths~$N$ exhibit approximately the same effective stiffness at low frequencies when fitted to the experimental data, suggesting that increasing~$N$ primarily improves dynamic fidelity without significantly altering the overall perceived stiffness.

\begin{table}[t!]
\centering \vspace{1mm}
\caption{System identification results for the FO-SLS models with $N \in \{51, 101, 151\}$.}
\vspace{-2mm}
\label{tab:hil_sysid_all_models}
\renewcommand{\arraystretch}{1.25}
\setlength{\tabcolsep}{6pt}
\footnotesize
\resizebox{.49\textwidth}{!}{$
\begin{tabular}{l c c c c c}
\hline
\textbf{Memory Lengths} & \(K_0\) & \(K_1\) & \(B_1\) & \(\alpha\) & NRMSE \\
\hline
FO-SLS N=51
& \(-3.14\) 
& \(5.60\) 
& \(5.83\) 
& \(0.246\) 
& \(0.604\%\) \\

FO-SLS N=101
& \(-2.89\) 
& \(5.70\) 
& \(5.89\) 
& \(0.203\) 
& \(0.314\%\) \\

FO-SLS N=151 
& \(-2.77\) 
& \(5.74\) 
& \(5.92\) 
& \(0.185\) 
& \(0.130\%\) \\

\hline
\end{tabular}
$}
\normalsize
\end{table}


\vspace{-3mm}
\subsection{Task}

The task was to compare the perceptual similarity between a reference (physical) viscoelastic material and the three virtual environments rendered by the haptic interface. In each trial of the cross-comparison experiment, participants interacted with the reference haptic device containing the physical viscoelastic material and with two virtual environments drawn from the set of three viscoelastic models. Their task was to evaluate the realism of each virtual environment and indicate which one felt most similar to the reference. Overall, participants were instructed to focus on the overall haptic impression, including stiffness, damping, and viscoelastic response, and to base their decisions solely on perceived similarity, without the aid of visual or auditory cues. No feedback on correctness was given; the task was purely perceptual and subjective.

\vspace{-3mm}
\subsection{Participants}

A total of 20 participants (mean age: 22.1 years) were recruited for the study. All participants were right-handed and reported no known sensory–motor impairments. Before participation, written informed consent was obtained from all volunteers in accordance with the ethical guidelines approved by the Institutional Review Board of Sabanci University (Protocol No: FENS-2025-17). None of the participants had extensive prior experience with haptic interfaces or psychophysical experiments.

\vspace{-3mm}
\subsection{Apparatus}
\label{sec:apparatus}

The experimental setup consisted of two identical haptic interfaces and a graphical user interface, as shown in Fig.~\ref{fig:Haptic Device}. One interface contained the physical viscoelastic material and served as the reference, while the other was controlled to render the three viscoelastic virtual environments using the system-identified parameter set with memory lengths $N \in \{51, 101, 151\}$. Participants held the devices with their dominant hands, palms facing inwards, and touched them with their index fingers in a palpation motion similar to that used in clinical diagnostics. 

\begin{figure}[t!]
    \centering
    \includegraphics[width=.7\linewidth]{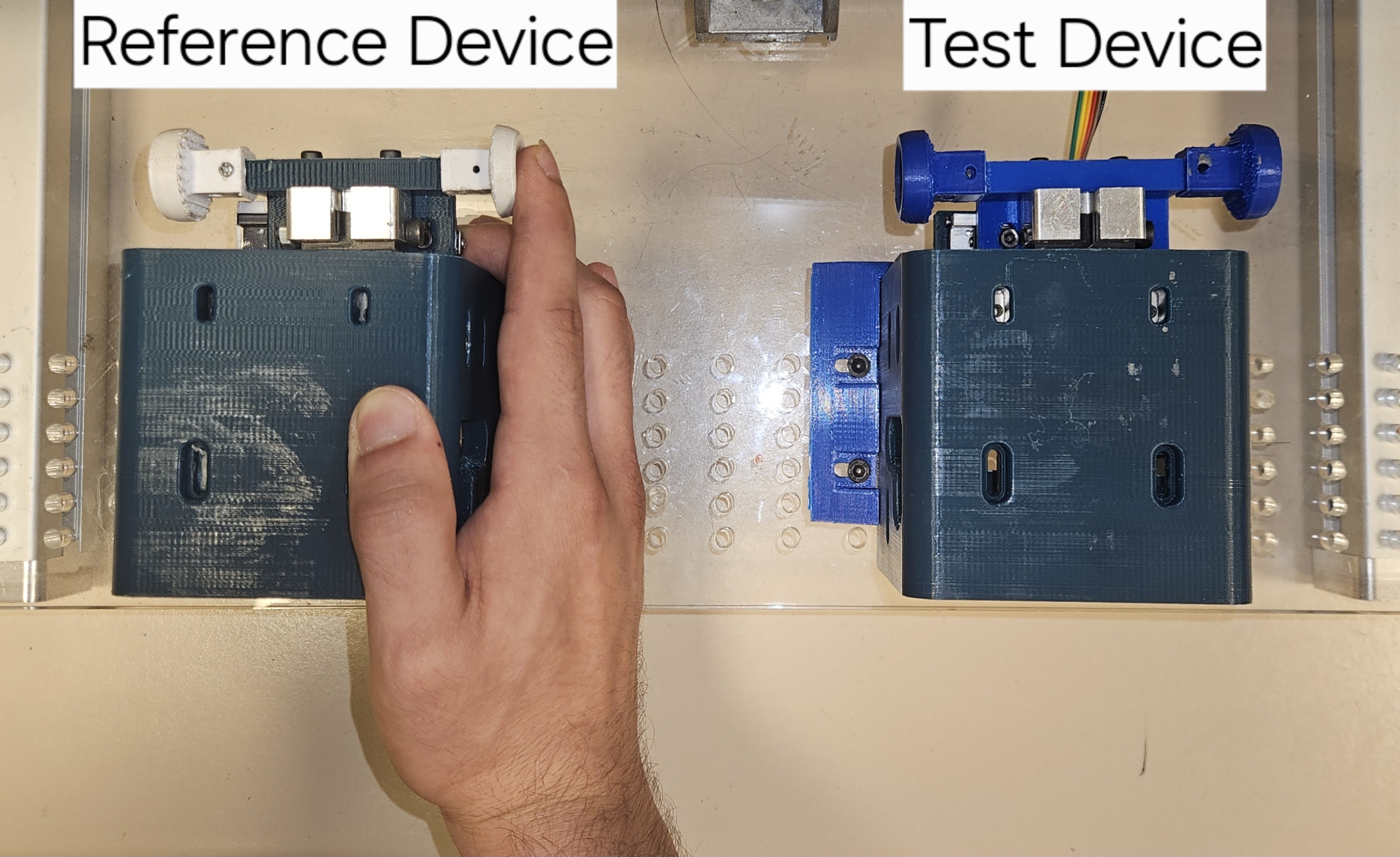}
    \vspace{-1mm}
    \caption{Two identical haptic interfaces.}
    \label{fig:Haptic Device} 
\end{figure}

\subsection{Experimental Protocol}

\subsubsection{Warm-Up}

Each participant performed a warm-up session preceding the actual comparison trials. In this warm-up, the participants were familiarized with the haptic devices and the interaction task. During this phase, subjects were presented with a set of different haptic renderings and asked to explore both the reference and virtual environments until they felt comfortable doing the task. The warm-up session lasted about five minutes.

\smallskip
\subsubsection{Cross-Comparison}

After the warm-up, participants proceeded to the cross-comparison session, which constituted the main experiment. The three FO-SLS models with memory lengths $N \in \{51, 101, 151\}$ were paired in all possible combinations of two. For each pair, participants performed multiple comparison trials.

In each trial, the presentation sequence of the two virtual environments was randomized. The participant first interacted with the reference device containing the real viscoelastic material and then engaged with each of the two candidate virtual worlds in turn. After exploration, the participant responded to the question: \emph{“Which of the two haptic interfaces felt more similar to the reference device?”}. In addition, participants are asked to evaluate each model independently based on its perceived similarity to the reference haptic interface. The following qualitative categories were provided:
\begin{itemize} \vspace{.5mm}
    \item \textbf{Close:} The tested haptic interface feels the same as, or highly close to, the reference haptic interface. \vspace{1.5mm}
    \item \textbf{Similar:} The tested haptic interface feels moderately similar to the reference haptic interface.  \vspace{1.5mm}
    \item \textbf{Different:} The tested haptic interface feels significantly different from the reference haptic interface.
\end{itemize}  The response was recorded through the GUI. This was repeated for all model pairs, with 12 repetitions per pair, resulting in a total of 36 comparative trials per subject. The complete cross-comparison session lasted about 20 minutes, and the whole experiment, including warm-up, was completed in about 25 minutes.

\subsection{Results and Discussion}

To evaluate the perceived realism of different short-memory discretization window sizes, we analyzed the classification and pairwise preference data collected during the haptic rendering fidelity validation sessions. The results are summarized by the classification confusion matrix in Fig.~\ref{fig:class_Pref}(a), while the overall pairwise preference rates are presented in Fig.~\ref{fig:class_Pref}(b). The virtual environments with $N=101$ and $N=151$ were mostly classified as ``close'' to the physical reference, whereas $N=51$ was mainly classified as ``similar''. Consistently, both $N=101$ and $N=151$ were dominantly preferred over $N=51$, while the direct comparison between $N=101$ and $N=151$ yielded similar preference rates. To further assess these trends, ordinal classifications and pairwise preferences were statistically analyzed separately.

\begin{figure}[t!]
    \centering
    \includegraphics[width=0.4\textwidth]{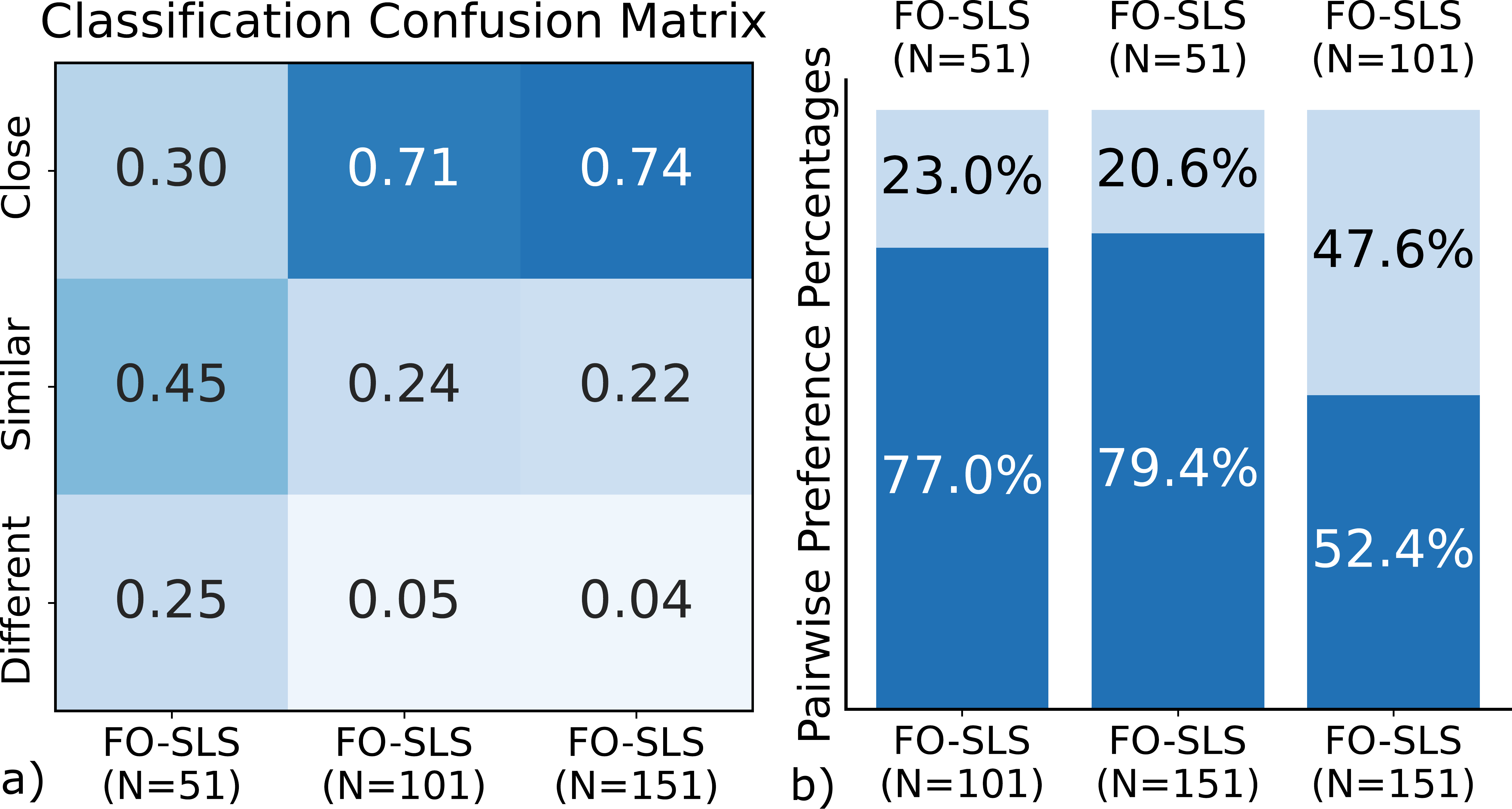} \vspace{-1mm}
    \caption{(a)~Normalized confusion matrix for FO-SLS models with different memory lengths. (b)~Pairwise preference percentages for FO-SLS models with different memory lengths.}
    \label{fig:class_Pref}
\end{figure}

The classification analysis tested whether each window size produced a realistic virtual environment by assessing the likelihood of being classified as ``close'' to the physical reference. Ordinal responses were recoded into a binary outcome, with ``close'' coded as 1 and ``similar'' or ``different'' coded as 0. Each rendering was analyzed independently using an intercept-only logistic regression model, which tested whether its ``close'' classification rate differed from the 50\% reference level without comparing renderings directly. The results are summarized in Table~\ref{tab:window_classification_stats}.

\begin{table}[t]
\centering
\caption{Close-classification analysis for different discretization window sizes.}
\label{tab:window_classification_stats} \vspace{-1mm}
\begin{tabular}{lcccc}
\hline
Rendering & Estimate & SE & $t$-value & $p$-value \\
\hline
FO-SLS ($N=51$)  & $-0.849$ & 0.125 & $-6.818$ & $<0.001$ \\
FO-SLS ($N=101$) & 0.916    & 0.172 & 5.323    & $<0.001$ \\
FO-SLS ($N=151$) & 1.036    & 0.130 & 7.983    & $<0.001$ \\
\hline
\end{tabular}\vspace{-1mm}
\end{table}

The preference analysis tested whether participants showed a systematic preference between renderings with different discretization window sizes. Preferences for each pairwise combination were analyzed independently using an intercept-only logistic regression model. The dependent variable indicated whether the rendering with the larger discretization window was preferred. Therefore, positive estimates indicate preference for the rendering with the larger window size, and the intercept tests whether this preference rate differed from 50\%. The results are summarized in Table~\ref{tab:window_preference_stats}. All statistical models used to analyze the classifications and preferences employed cluster-robust standard errors to account for repeated responses within participants.

\begin{table}[t]
\centering
\caption{Pairwise preference analysis for different discretization window sizes.} \vspace{-1mm}
\label{tab:window_preference_stats}
\begin{tabular}{lcccc}
\hline
Comparison & Estimate & SE & $t$-value & $p$-value \\
\hline
$N=51$ vs. $N=101$  & 1.207 & 0.155 & 7.773 & $<0.001$ \\
$N=51$ vs. $N=151$  & 1.347 & 0.204 & 6.616 & $<0.001$ \\
$N=101$ vs. $N=151$ & 0.095 & 0.142 & 0.669 & 0.503 \\
\hline
\end{tabular}
\end{table}

The classification results showed a clear distinction between the smallest window size and the two larger window sizes. The rendering with $N=51$ was classified as close significantly less often than the 50\% reference level ($\beta=-0.849$, $p<0.001$), whereas both $N=101$ and $N=151$ were classified as close significantly more often than this reference level ($\beta=0.916$ and $\beta=1.036$, respectively; both $p<0.001$). Thus, the two larger window sizes were consistently judged to be close to the physical reference, whereas the smaller window size was not. The preference results followed the same pattern: both $N=101$ and $N=151$ were significantly preferred over $N=51$ ($\beta=1.207$ and $\beta=1.347$, respectively; both $p<0.001$). In contrast, the direct comparison between $N=101$ and $N=151$ did not show a significant preference advantage for the larger window size ($\beta=0.095$, $p=0.503$).

These results suggest that increasing the discretization window from $N=51$ to $N=101$ led to a clear improvement in perceived realism, whereas increasing the window further from $N=101$ to $N=151$ did not provide evidence of an additional perceptual benefit, as both of the discretization window sizes yield comparably realistic rendering. The classification estimates for $N=101$ and $N=151$ were numerically similar, and the direct pairwise preferences did not show a significant preference advantage for $N=151$ over $N=101$. Both window sizes achieved high perceived realism, while providing no statistical evidence that selecting window size as $N=151$ produced perceptually more realistic renderings than selecting the window size as $N=101$. 


\vspace{-2mm}

\section{Conclusions}
In this paper, we investigated the use of the viscoelastic models with FO terms for haptic rendering. The problem was approached from two complementary perspectives: theoretical analysis and perceptual evaluation.

First, we presented rigorous passivity analysis to ensure the coupled stability of the FO-SLS model under finite-memory discretization. We derived explicit closed-form passivity conditions and showed that the result provides a unified formulation from which the passivity results for the commonly utilized viscoelastic models can be readily derived as special cases. We have also experimentally verified the established bounds. In this respect, the proposed passivity condition generalizes existing results and promotes the safe, systematic use of fractional-order models, which are less commonly used, possibly due to coupled stability and finite-memory concerns.

We also obtained closed-form expressions for the effective stiffness and damping of the FO-SLS model under finite-memory discretization. These results enable a thorough analysis of the frequency-dependent behavior of viscoelastic rendering and permit examination of model performance at critical frequencies. We demonstrated that the resulting formulations represent generalized expressions that recover the classical equations of commonly used viscoelastic models under appropriate parameter substitutions. This unifying perspective provides additional insight into the relationship between fractional- and integer-order representations and underscores the added flexibility that fractional calculus offers in modeling.

Finally, we investigated whether the length of memory terms used during fractional-order modeling results in practically meaningful differences for human operators while interacting with viscoelastic virtual environments. In this regard, we conducted a comparative study between a physical reference viscoelastic model and the three FO-SLS models with different memory lengths. The experimental results revealed that the fractional-order models with sufficiently large memory were statistically distinct from the models with short memory, but this effect saturates after a memory limit. Our results support the suggestion that sufficiently large memory lengths need to be considered while rendering fractional-order viscoelastic models to provide a realistic haptic experience for the users.

Overall, these results are consistent with a saturation-like trend, in which the perceptual benefit of increasing the memory length becomes limited beyond approximately $N=101$ under the present experimental conditions. At the same time, the high close-classification rates observed for both $N=101$ and $N=151$ indicate that realistic FO-SLS rendering can be achieved with sufficiently large discretization windows. From a practical rendering perspective, this supports using $N=101$ as a suitable discretization window, as it maintains high perceived realism while avoiding the additional computational cost of a longer memory window.


\vspace{-2mm}
\section*{Acknowledgment}

This work has been partially supported by the TUBITAK Grant~23AG003. 

\vspace{-1.5mm}
\bibliographystyle{IEEEtran}
\bibliography{Bibliography}

\vspace{-1\baselineskip}
\begin{IEEEbiography}[{\includegraphics[width=.8in,height=1.2in,clip,keepaspectratio]{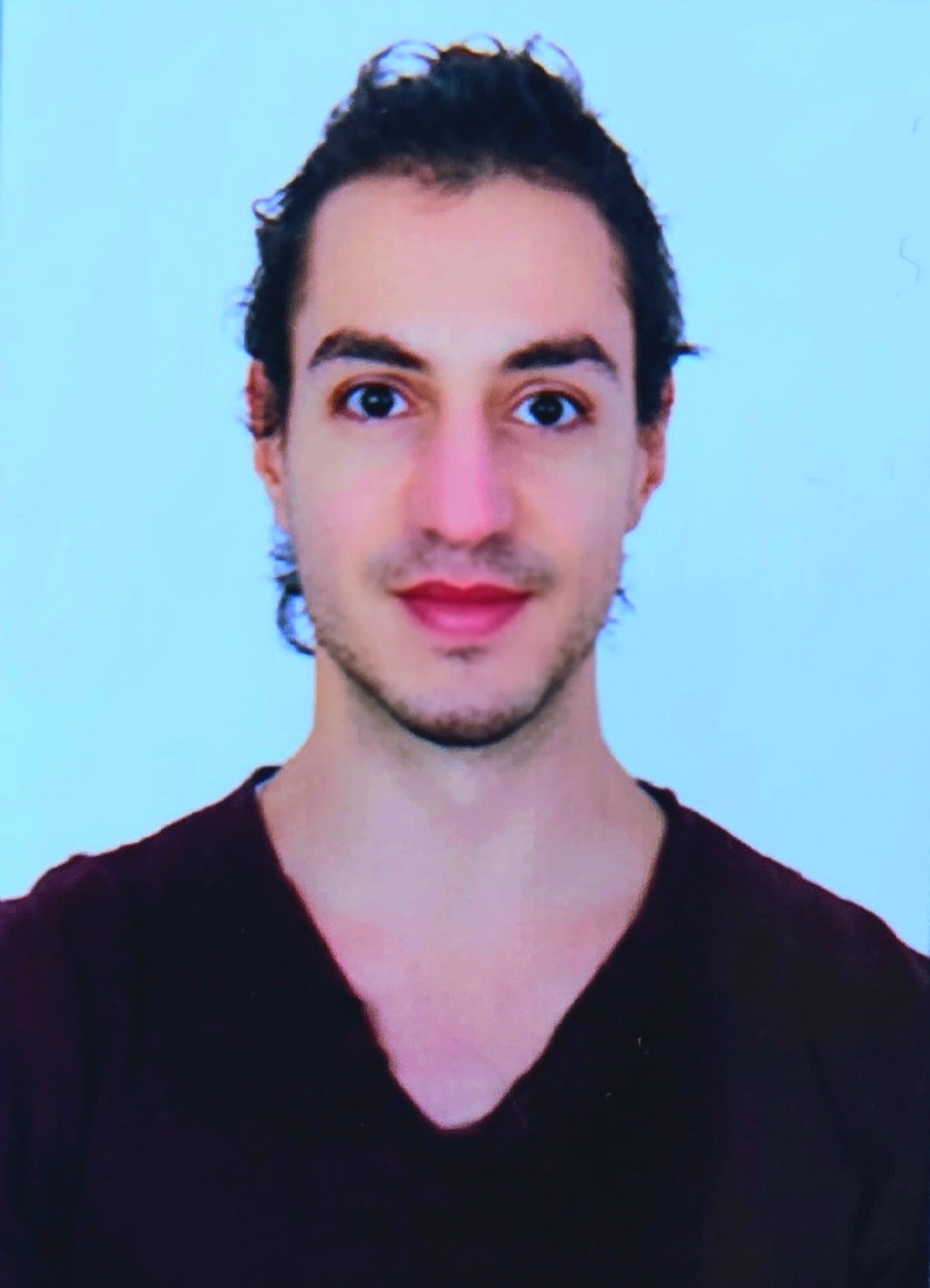}}]
{Gorkem Gemalmaz}  received his B.Sc. degree in mechanical engineering from Middle East Technical University~(2022). Currently, he is pursuing his Ph.D. degree at Sabanci University.
His research interests include physical human-robot interaction and interaction control.
\end{IEEEbiography}

\vspace{-1\baselineskip}
\begin{IEEEbiography}[{\includegraphics[width=.8in,height=1.2in,clip,keepaspectratio]{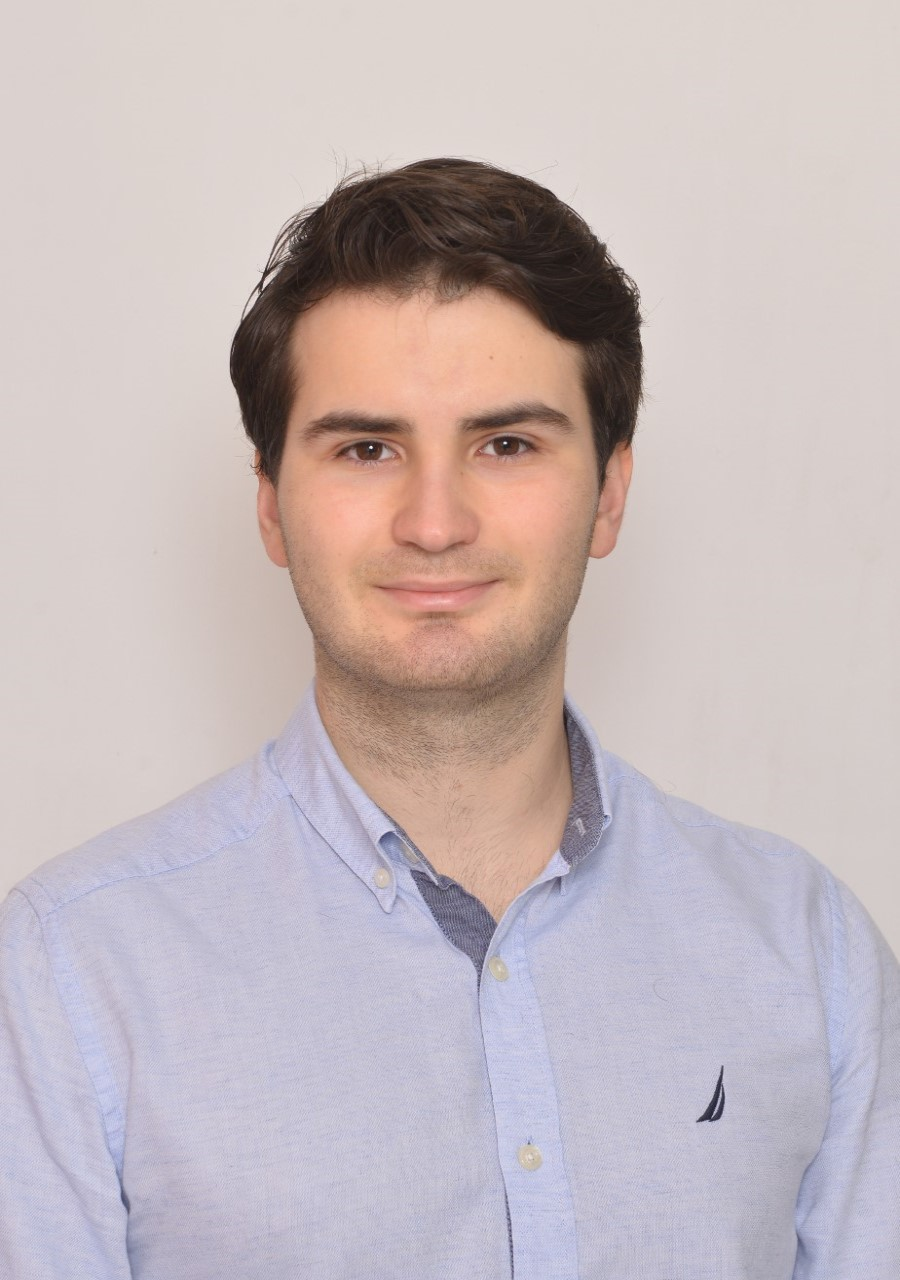}}]
{Harun Tolasa}  received his B.Sc. degree in mechanical engineering from Bilkent University (2021) and his M.Sc. in mechatronics engineering from Sabanci University~(2024). Currently, he is pursuing his Ph.D. degree at Sabanci University. His research interests include active learning, human-in-the-loop optimization, and haptic rendering.
\end{IEEEbiography}

\vspace{-1\baselineskip}
\begin{IEEEbiography}[{\includegraphics[width=1in,height=1.2in,clip,keepaspectratio]{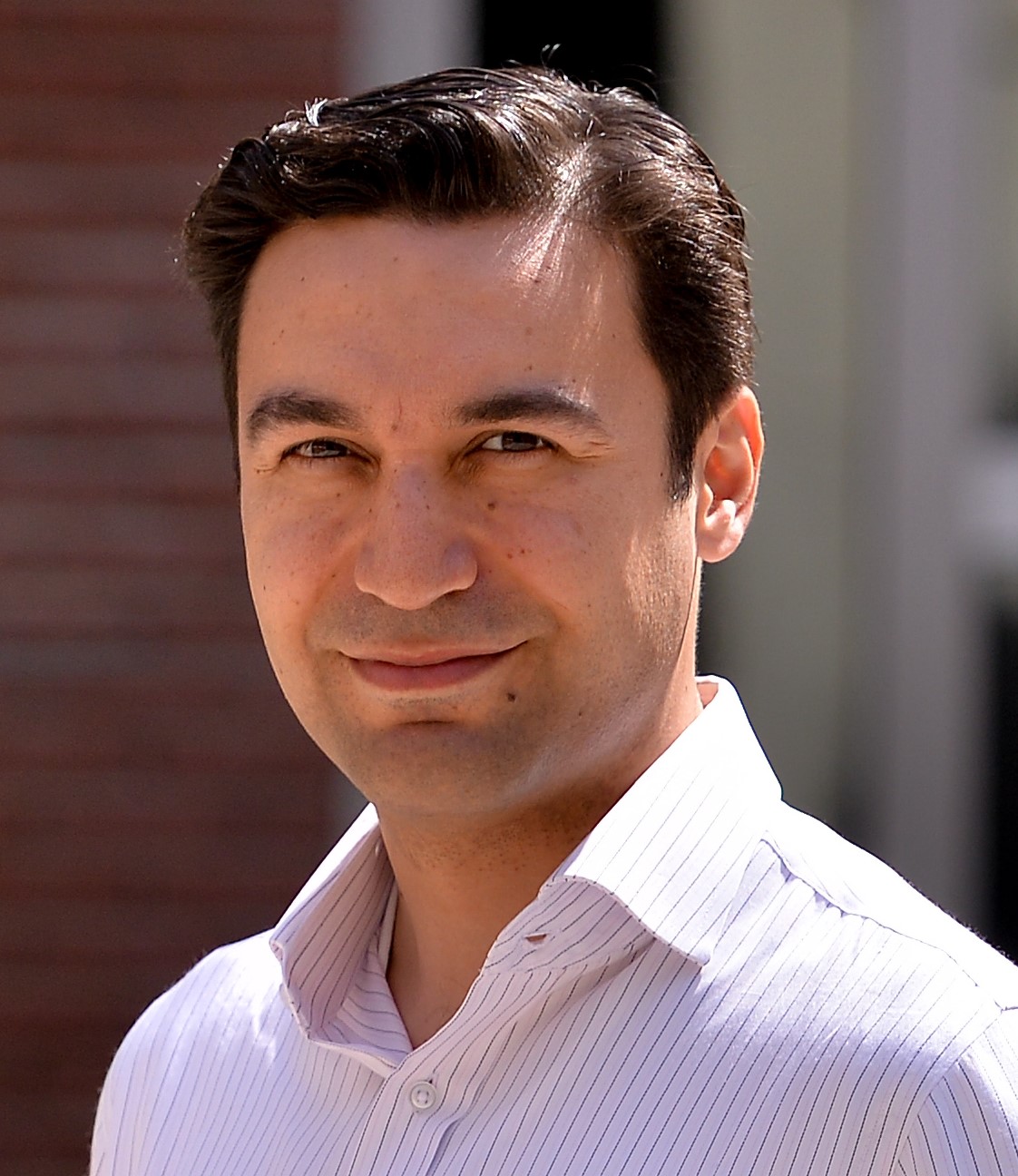}}]
{Volkan Patoglu} is a professor in mechatronics engineering at Sabanci University.
He received his Ph.D. in mechanical engineering from the University of Michigan, Ann Arbor~(2005) and worked as a postdoctoral researcher at Rice University~(2006). His research is in the area of physical human-machine interaction, in particular, the design and control of force-feedback robotic systems with applications in rehabilitation. His research extends to cognitive robotics. He has served as associate editor for IEEE Transactions on Haptics~(2013--2017), IEEE Transactions on Neural Systems and Rehabilitation Engineering~(2018--2023), and IEEE Robotics and Automation Letters~(2019--2024).
\end{IEEEbiography}

\vfill

\end{document}